%% file: main.tex
\documentclass[conference]{IEEEtran}
\IEEEoverridecommandlockouts
\usepackage{amsmath,amssymb,amsfonts}
\usepackage{algorithmic}
\usepackage{graphicx}
\usepackage{textcomp}
\usepackage{xcolor}
\usepackage{amsthm}
\usepackage{algorithm}
\usepackage{todonotes}
\usepackage[hyphens,spaces,obeyspaces]{url}
\PassOptionsToPackage{hyphens}{url}\usepackage{hyperref}
\usepackage{multirow, fullpage, mathtools, longtable, rotating, makecell, array}
\usepackage[backend=biber,style=numeric]{biblatex}
\DeclareMathOperator*{\argmax}{argmax}
\DeclareMathOperator*{\argmin}{argmin}

\theoremstyle{definition}

\usepackage[all]{xy}
\xyoption{poly}
\xyoption{arc}

\newcommand{\N}{\mathbb{N}}

\def\ba #1\ea{\begin{align} #1 \end{align}}
\def\bas #1\eas{\begin{align*} #1 \end{align*}}
\def\bml #1\eml{\begin{multline} #1 \end{multline}}
\def\bmls #1\emls{\begin{multline*} #1 \end{multline*}}

\newcommand{\cS}{\mathbf{S}}
\newcommand{\cA}{\mathbf{A}}

\newtheorem{thm}{Theorem}[section]

\newtheorem{lem}[thm]{Lemma}
\newtheorem{cor}[thm]{Corollary}

\theoremstyle{remark}

\theoremstyle{definition}
\newtheorem{dfn}[thm]{Definition}

\begin{document}

\title{Structural Similarity for Improved Transfer in Reinforcement Learning \\
\thanks{This research was funded by the DARPA Lifelong Learning Machines (L2M) program.}
}

\author{\IEEEauthorblockN{C. Chace Ashcraft}
\IEEEauthorblockA{\textit{Research and Exploratory Development} \\
\textit{Johns Hopkins University Applied Physics Lab}\\
Laurel MD, USA \\
chace.ashcraft@jhuapl.edu}
\and
\IEEEauthorblockN{Benjamin Stoler}
\IEEEauthorblockA{\textit{Research and Exploratory Development} \\
\textit{Johns Hopkins University Applied Physics Lab}\\
Laurel MD, USA \\
benjamin.stoler@jhuapl.edu}
\and
\IEEEauthorblockN{Chigozie Ewulum}
\IEEEauthorblockA{\textit{Research and Exploratory Development} \\
\textit{Johns Hopkins University Applied Physics Lab}\\
Laurel MD, USA \\
chigozie.ewulum@jhuapl.edu}
\and
\IEEEauthorblockN{Susama Agarwala}
\IEEEauthorblockA{\textit{Research and Exploratory Development} \\
\textit{Johns Hopkins University Applied Physics Lab}\\
Laurel MD, USA \\
susama.agarwala@jhuapl.edu
}
}

\maketitle

\begin{abstract}

Transfer learning is an increasingly common approach for developing performant RL agents. However, it is not well understood how to define the relationship between the source and target tasks, and how this relationship contributes to successful transfer. We present an algorithm called \emph{Structural Similarity for Two MDPS}, or \emph{SS2}, that calculates a state similarity measure for states in two finite MDPs based on previously developed bisimulation metrics, and show that the measure satisfies properties of a distance metric. Then, through empirical results with GridWorld navigation tasks, we provide evidence that the distance measure can be used to improve transfer performance for Q-Learning agents over previous implementations.

\end{abstract}


\section{Introduction} \label{sec:intro}

The last decade has seen impressive results in reinforcement learning (RL), particularly for deep reinforcement learning (DRL) \cite{dqn2015, openai2019dota, alphastar}. However, many of these algorithms tend to be computationally expensive and produce policies that are brittle to shifts in their deployment distributions. One potential solution to both of these problems is to leverage transfer learning \cite{carroll2005task, pan2009survey, zhu2020transfer}, i.e., using previously trained policies to jumpstart learning in a new task, or by developing curricula \cite{curriculum, narvekar2020curriculum} of tasks to facilitate the learning of advanced skills or behaviors. 

To best leverage transfer, the task being transferred from should be \emph{similar} to the task being transferred to. Yet what makes one task similar to another in RL, aside from good transfer \cite{JMLR:v10:taylor09a, visus2021taxonomy}, is currently not well understood, nor how to exploit such similarity to facilitate transfer learning.



In this work, we propose a new approach, based on \cite{wang2019} and \cite{song2016}, to estimate state and action similarities between two tasks (where each task is considered to be a Markov Decision Process, or MDP), and then leverage those similarities for knowledge transfer in Q-learning. We show that our approach, referred to as \emph{Structural Similarity for Two MDPs} or SS2, satisfies the properties of a similarity metric, compare SS2-based knowledge transfer to prior methods \cite{song2016}, and propose ways in which the prior methods can be augmented by SS2. We assess task similarity and knowledge transfer on a set of gridworld RL tasks, and show that the SS2-augmented methods outperform all other tested methods. 

We also discuss some observations regarding task distances such as those proposed by \cite{song2016}, limitations of our proposed algorithms, and possible future work. Notably, reducing a set of state similarities into a single value for task similarity, as done in \cite{song2016} using the Hausdorff and Kantorovich distances, did not correlate well with transfer in our experiments.

\section{Related Work}

Task similarity measures tend to fall into two categories: model-based and performance-based \cite{visus2021taxonomy}. 

\emph{Model-based metrics} tend to focus on similarities in the structures of the MDPs, e.g. bisimulation metrics \cite{ferns2004metrics}. Song et al. \cite{song2016} introduced a metric $d'$ that  extends this bisimulation metric to compare similarity between separate MDPs \cite{song2016}, but their metric requires that two MDPs be homogeneous, i.e., having equivalent state representations, a one-to-one correspondence between action spaces, and by satisfying a condition they refer to as \emph{reward-linked}. Their metric iteratively builds a matrix of distances between states by adding the magnitude of the difference of expected rewards of actions between states and the Kantorovich distance, or earth mover's distance (EMD), between the probability distributions of the corresponding actions between the two states. They then propose algorithms that leverage $d'$ to transfer knowledge from a Q function learned from one task to initialize a Q table for another task to accelerate learning. The Q table of the new task is initialized either with Q values from similar states in the trained task (what the authors refer to as the ``state method''), or with a sum of Q values from similar states weighted by their similarity as computed using $d'$ (the ``weight method''). For simplicity, we instead refer to these algorithms as \texttt{T-STATE} and \texttt{T-AVG}, respectfully. We discuss these algorithms in more detail in Section \ref{sec:experiment}.

Wang et. al. \cite{wang2019} proposed a similar solution to quantify the similarities between states and actions within a single MDP by utilizing the Hausdorff distance between out-neighbors of states and the EMD between out-neighbors of actions in a bipartite MDP graph. Their solution generates two matrices of similarities, one for state similarities and one for action similarities, but for states and actions within a single MDP. 

In \cite{zhang2020learning}, Zhang et al. also utilized a bisimulation metric in order to disassociate task-irrelevant information between observations within MuJoCo~\cite{todorov2012mujoco} tasks and high-fidelity highway driving tasks using CARLA, an open urban driving simulator \cite{dosovitskiy2017carla}.

\emph{Performance-based metrics}, on the other hand, aim to find similarities based on transfer performance between tasks. Carrol et. al. \cite{carroll2005task} proposes a task library system as a part of the \emph{lifelong learning} \cite{thrun1998lifelong} paradigm that allows an agent to improve its learning ability as it is exposed to successive tasks through task localization, similarity discovery, and task transfer. They propose multiple task similarity measures: similarity based on rewards, similarity based on number of states with identical policy, mean square error between Q-values, and mean squared error between rewards (R values). Castro et. al. \cite{castro2019scalable} expanded upon the preexisting bisimulation metrics by creating two metrics that are less computationally intensive, with one having guaranteed convergence and the other being able to learn an approximation for continuous state MDPs. However, these require an actively learning policy to operate on the MDP in order to compute their measures, therefore we categorize it here. Other examples in this vein include \cite{mahmud2013clustering} and \cite{castro2021mico}.  

Some methods additionally include the learning of the metric itself. Works such as \cite{ammar2014automated} and \cite{serrano2021inter} propose metrics that require training models to extract useful information on source and target tasks. Serrano et. al.'s method performs operations on Q tables of source and target tasks, while Ammar et. al.'s method utilizes restricted Boltzmann machines on the source and target tasks, allowing a modelling of the behavioural dynamics of the two MDPs to be used as a similarity measure on tasks within the domain the measure was trained on.

\section{Preliminaries}

\subsection{MDPs and MDP Graphs}

In this work we primarily follow the definitions and notation from Wang et. al. \cite{wang2019}, so we define an MDP as $M=(S, A, P, R)$, where $S$ is a finite set of states, $A$ is a finite set of actions, $P: S \times A \times S \longrightarrow [0, 1]$ is the transition function providing the probabilities of an agent going to a state $s' \in S$ given it is currently in state $s \in S$ and took action $a \in A$, and finally, $R: S \times A \times S \longrightarrow [0, 1]$ is the reward function for taking action $a$ while in state $s$ and transitioning to $s'$. We note that, while common definitions of MDPs include a discount factor, usually denoted $\gamma$, as with the Song and Wang approaches, we do not make explicit use of this value. The same can be said for the initial agent state $s_0 \in S$. 

We define the MDP graph of an MDP \\$M=(S,A,P,R)$ as $G_M=(V,\Lambda,E,\Psi,p,r)$. $G_M$ is a heterogeneous directed bipartite graph with two types of nodes, \emph{state} nodes and \emph{action} nodes. $V$ denotes the set of state nodes, and $\Lambda$ denotes the set of action nodes. $E$ is the set of decision edges from state nodes to action nodes, and $\Psi$ is the set of transition nodes from action nodes to state nodes. Transition edges are weighted by the transition probabilities $p$ and by rewards $r$. Note that there is a one-to-one correspondence between an MDP and its constructed graph as described above. For more details and facts about constructing an MDP graph from an MDP, see \cite{wang2019}, as well as our experiment methodology in Section \ref{sec:experiment}.

\subsection{Reinforcement Learning and Transfer}

The aim of reinforcement learning is to find an optimal policy $\pi*$ for a given MDP $M$, utilizing the information gained by interacting with the MDP, particularly by discovering rewards in the MDP and maximizing them. Again following notation from \cite{song2016}, given an MDP $M = (S, A, P, R)$, we define a policy $\pi: S \times A \longrightarrow [0, 1]$ as a mapping from state, action pairs to reward values in the interval $[0, 1]$. The optimal policy can be characterized as

\begin{align*}
    \pi^* &= \argmax_{\pi} E_{\pi} \left[ \sum_{t=0}^\infty \gamma^t r_t | s_0 = s \right] \\
          &= V^{\pi}(s) 
\end{align*}
where $\gamma$ is a discount factor, $r_t$ is the reward at time $t$, and $s_0$ is the starting state of $M$ ($\gamma$ and $s_0$ are often included in the definition of the MDP, but in this case we omit them). $V^{\pi}: S \longrightarrow \mathcal{R}$ is called the \emph{value function} of the policy $\pi$, and is the expected cumulative reward of following $\pi$ from state $s$ onward. Deviating slightly in notation, we say $P(s, a, s')$ is the probability of transitioning from state $s$ to state $s'$ given that agent takes action $a$.

\section{Structural Similarity for Two MDPs (SS2)}

We define our distance between two given MDPs mostly following the recursion given in \cite{wang2019}'s implementation, which we will call SS1 for the remainder of this section, but with some notable differences. We construct a similarity matrix from an MDP graph composed of two disconnected components, each component being the MDP graph of the respective MDP. We also impose some additional restrictions on the initial conditions of SS1's recursion, and then use only a submatrix of the final similarity matrix for the final distance measure; doing so allows us to establish properties of that sub-matrix that are desirable, as shown in \cite{wang2019}. In Section \ref{sec:alg-simp} we provide a simplification of this algorithm wherein we can  compute the sub-matrix directly. For the remainder of this work, we will refer to this metric as Structural Similarity for two MDPs\footnote{In theory, this could be extended to more than two MDPs, but we only consider two in this work.}, or SS2, for short.

\subsection{Foundation}

Let $M=(S_M, A_M, P_M, R_M)$ and $N=(S_N, A_N, P_N, R_N)$ be two MDPs. We represent both as MDP graphs $G_M = (V^M, \Lambda^M, E^M, \Psi^M, p^M, r^M)$ and $G_N = (V^N, \Lambda^N, E^N, \Psi^N, p^N, r^N)$ as established for SS1\footnote{It is worth noting here that while the action space is a finite size, say $\sigma$, each action available to the agent at each state is its own unique node. So if all $\sigma$ actions are available to the agent at each state, there would be $\sigma \cdot |S|$ total action nodes.}. Then, let $G = G_M \sqcup G_N = (V, \Lambda, E, \Psi, p, r)$ be the disjoint union of the two graph MDPs. This represents a situation where there are two sets of states, $V^M$ and $V^N$, such that if one starts in one set, there are no sets of actions such that one can land in a state in the other set. We desire to establish the state and action distance measures:

\begin{align}
    \delta_{\hat{S}}(u, v) \vcentcolon= 1 - \sigma_{\hat{S}}(u, v) \quad \forall u, v \in V 
\end{align}
\begin{align}
    \delta_{\hat{A}}(u, v) \vcentcolon= 1 - \sigma_{\hat{A}}(\alpha, \beta) \quad \forall \alpha, \beta \in \Lambda 
\end{align}
where $\sigma_{\hat{S}}$ and $\sigma_{\hat{A}}$ are the state and action similarities for state space ${\hat{S}} = S_M \cup S_N$ and action space ${\hat{A}} = A_M \cup A_N$, respectfully. In practice, these functions will be defined as matrices, i.e. $\sigma_S(u, v) = \cS [u_{idx}, v_{idx}]$, where $\cS$ is a matrix, $u_{idx}$ and $v_{idx}$ denote the row and column of $\cS$ representing $u$ and $v$, respectively, and $\cS[i,j]$ is the $(i,j)^{th}$ entry of $\cS$. For $G$ as defined above, we see that $\sigma_{\hat{S}} = \cS^G$ takes on a useful form.

Let $\cS^M$ and $\cS^N$ be similarity matrices for the states of $M$ and $N$ respectively. Let $\cS^{M,N}$ (respectively, $\cS^{N,M}$) be the similarity matrix between the two MDPs with rows corresponding to the states of $M$ and the columns the states of $N$ (respectively, rows corresponding to states of $N$ and columns corresponding to states of $M$). Similarly we define matrices $\cA^M$, $\cA^N$, $\cA^{M,N}$ and $\cA^{N,M}$ with rows and columns corresponding to states in $M$ and states in $N$ as appropriate. Then one may write the comparison matrix of G as follows: \bas \cS^G = \begin{bmatrix} \cS^M & \cS^{M,N} \\ \cS^{N,M} & \cS^N \end{bmatrix} \quad ; \quad \cA^G = \begin{bmatrix} \cA^M & \cA^{M,N} \\ \cA^{N,M} & \cA^N \end{bmatrix} \eas 

It is worth noting here that the sub-matrix $\cS^{M,N}$ represents the distances between the state nodes of the MDP graph of M and MDP graph of N, and similarly with $\cA^{M,N}$ and the action nodes in each MDP graph. Thus, the real final product we are truly interested in is $\cS^{M,N}$, but for now we consider the computation of $\cS^G$ in order to leverage results from SS1 and make the argument that $\cS^G$ is, in fact, a distance metric for MDP states. 

\subsection{Base Cases}

For $u, v \in \cS^G$, we set the similarity of a node to itself as 1. Similarly, two absorbing nodes are maximally similar, while absorbing and non-absorbing nodes are as dissimilar as possible. In other words, let $\cS^{G, t}_{u, v}$ denote the $(u, v)^{th}$ entry of the similarity matrix for graph $G$ at time step $t$, then 
\bas \cS^{G,0}_{u,v} =  \begin{cases}1 & u = v \\ 
                                                      \omega  & N_u = N_v = \emptyset \\ 
                                                      0 &  N_u = \emptyset \textrm{ xor } N_v = \emptyset \\
                                                      d_{uv} \in [0,1] & u \neq v \end{cases}  \eas%
where $\omega \in (0, 1]$ denotes the maximal similarity value assigned to any two \emph{different} states. For example, if two states are exactly the same but are in two different connected components of the graph, their similarity would be $\omega$. 

As with SS1, we allow freedom in the initialization of similarity between non-absorbing pairs; however, we also impose one additional requirement. We insist that $\cS^G$ be initialized \emph{such that the entries in the matrix $1-\cS^{G,0}$ define a distance function on the nodes}. Namely, we insist that $d_{uv} = d_{vu}$ and that for all triples $u, v, w \in \cS_G$, $\big(1-\cS^{G,0}_{u,v} \big) + \big(1-\cS^{G,0}_{v,w}\big) \geq \big(1-\cS^{G,0}_{u,w} \big)$. We impose these constraints in order to ensure convergence of the recursion and to ensure $\cS$ and $\cA$ represent distances (more details are provided in the proof of Theorem \ref{res:goodinitcond}). We follow the same procedure for $\cA^G$.

\subsection{Distance equations}

In this section we restate the distance functions on which SS1 and our recursion algorithms depend. Namely, a reward distance $\delta_{rwd}$, the Hausdorff distance $\delta_{Haus}$, and the earth mover's distance $\delta_{EMD}$. Let $N_x$, for $x \in V^M \cup \Lambda^M \cup V^N \cup \Lambda^M$, denote the out neighbors of $x$, then, for action node $\alpha$ and another set of action nodes $N$, define $\delta_A(\alpha, N) \vcentcolon= \min_{\beta \in N} \delta_A(\alpha, \beta)$. The previously listed distance equations are then defined as follows:

\begin{equation}
\begin{aligned}
    \delta_{rwd} = |\mathbb{E}[r_\alpha] - \mathbb{E}[r_\beta]|
\end{aligned}
\end{equation}
\begin{equation}
\begin{aligned}
    \delta_{Haus}(u, v; \delta_A) = \max_{\substack{\alpha \in N_u \\  \beta \in N_v }} \{\delta_A(\alpha, N_v), \delta_A(\beta, N_u)\}
\end{aligned}
\end{equation}
\begin{equation}
\begin{aligned}
    \delta_{EMD}(p_\alpha, p_\beta; \delta_S)=& \min_{\mathbf{F}} 
                                                \sum_{u \in N_\alpha} \sum_{v \in N_\beta}
                                                f_{u, v} \delta_S(\alpha, \beta)\\%
                &\textrm{s.t.} \: \: \: \forall u, v \in V : f_{u,v} \geq 0,\\%
                &\quad \ \ \, \forall u \in V : \sum_{v \in V} f_{u,v} = p(\alpha, u), \\
                &\quad \ \ \, \forall v \in V : \sum_{u \in V} f_{u,v} = p(\beta, v).
\end{aligned}
\end{equation}

With this distance functions defined, we can restate the equations used in SS1's recursion. Choose values $C_S, C_A \in (0, 1)$, where $C_S$ can be thought of as an upper bound on state similarities given their action similarities are equal, and $C_A$ can be thought of as a weighting factor between between magnitudes of contribution of reward differences and state transition differences to final distance between actions. The state updates are given by the following equations:

\begin{equation}
    \label{eq:state_sim}
    \begin{aligned}
    \cS_{u, v}^{G, i} = C_S(1 - \delta_{Haus}(u, v; 1 - \cA^{G, i-1}))
    \end{aligned}
\end{equation}
and
\begin{equation}
    \label{eq:action_sim}
    \begin{aligned}
    \cA^{G,i}_{\alpha, \beta} = 1 &- (1 - C_A)\delta_{rwd}(\alpha,\beta) \\
             &- C_A\delta_{EMD}(\alpha, \beta; 1-\cS^{G,i-1}).
    \end{aligned}
\end{equation}

\subsection{Algorithm}

Combining definitions from previous sections, we are able to state our algorithm for determining inter-state distances as Algorithm \ref{alg:song-based}. While we do not specify any specific convergence criteria, it is worth noting that the exact results produced Algorithm \ref{alg:song-based} will vary based on the criteria chosen. However, iterating the name number of times produces consistent results. 

The algorithm requires as input the specification of the graph MDP $(V, \Lambda, E, \Psi, p, r)$, values for $C_S, C_A$, and the initialized values of $S^G$, $A^G$ as previously described. For the remainder of this document, we assume the value of $\omega$ for each $S^{G,0}$, $A^{G, 0}$ is set to $C_S$. Initializing $S^G$ and $A^G$ this way ensures convergence to a unique solution as shown in the proof of Lemma \ref{res:goodinitcond} and the following Corollary. 

This choice of $\omega$ is further justified by considering the output of the Hausdorff distance when the nodes in the input have no out-neighbors. This results in the distance between nodes in two empty sets, which in theory, is trivially $0$. Therefore, the outcome of Equation \ref{eq:state_sim} will always be $C_S$. A similar argument can be made for the case where only one set is non-empty, in which case the Hausdorff distance is theoretically infinity, or in our case, the max distance of 1, which results in Equation \ref{eq:state_sim} being $0$. 

Finally, we emphasize that the range of the reward function must be constrained to be within $[0,1]$. For tasks where this is not the default reward scheme, the reward function can simply be normalized to fit into this interval for the purpose of this similarity (or distance) computation.

\begin{algorithm}
\caption{Structural Similarity for two MDPs (\texttt{SS2})}\label{alg:song-based}
    \begin{algorithmic}[1]
    \REQUIRE Graph MDP = $(V, \Lambda, E, \Psi, p, r)$
    \REQUIRE $C_s, C_a \in (0, 1)$, $S^{G,0}$, $A^{G, 0}$
    \ENSURE $(1-\cS_{u, v}^{G, 0}) + (1-\cS_{v, w}^{G, 0}) \geq (1-\cS_{u, w}^{G, 0})$; \\ 
            $\qquad \qquad \qquad \qquad \qquad \qquad \qquad \quad \forall u,v,w \in V$
    \ENSURE $(1-\cA_{u, v}^{G, 0}) + (1-\cA_{v, w}^{G, 0}) \geq (1-\cA_{u, w}^{G, 0})$; \\ 
            $\qquad \qquad \qquad \qquad \qquad \qquad \qquad \quad \forall u,v,w \in V$
    \REPEAT
        \FORALL{$\alpha \in N_u, \beta \in N_v$: $u, v \in V$; $u \neq v$}
            \STATE $\cA^{G,i}_{\alpha, \beta} \leftarrow 1-(1-C_A)\delta_{rwd}(\alpha,\beta)$ \\
            $\qquad \qquad \:\, - \; C_A\delta_{EMD}(\alpha, \beta; 1-\cS^{G,i-1})$
        \ENDFOR
        \FORALL{$u, v \in V$; $N_u \neq \emptyset$; $N_v \neq \emptyset$}
            \STATE $\cS_{u, v}^{G, i} \leftarrow C_S(1 - \delta_{Haus}(u, v; 1 - \cA^G))$
        \ENDFOR
    \UNTIL $\cS^G$ and $\cA^G$ converge 
    \RETURN $(\cS^{G, *}, \cA^{G, *})$ 
    
    \COMMENT{$\cS^{G, *}$ = state similarity matrix, 
    $\cA^{G, *}$ = action similarity matrix}
    \end{algorithmic}
\end{algorithm}

\subsection{Comparison with Original Algorithm}

Before proceeding, we make a few remarks on Algorithm \ref{alg:song-based}, in particular on how it differs from SS1. 

In Algorithm \ref{alg:song-based} we insist that the matrices $\cA^{G,0}$ and $\cS^{G,0}$ are initialized such that $1-\cA^{G,0}$ and $1-\cS^{G,0}$ are both distances. This is because, in the iterative step, the entries of $\cA^{G,i}$ and $\cS^{G,i}$ are defined by distance functions that take as an input the distance matrix $1-\cA^{G,i-1}$ or $1-\cS^{G,i-1}$. In other words, if $1-\cA^{G,i-1}$ and $1-\cS^{G,i-1}$ distances, then both $1-\cA^{G,i}$ and $1-\cS^{G,i}$ are guaranteed to be distances. This is obvious for $1-\cA^{G,i}$. It remains to check that if $1-\cS^{G,i-1}$ defines a distance on the non-absorbing nodes, then this extends to a distance on all nodes in $\cS_G$. For this, we need only check the triangle inequality. Let $\circ$ indicate an absorbing node and $*$ a non-absorbing node. The triangle inequality can be seen from the following diagrams \bas 
{\xymatrix{
 & {\ast} \ar@{<->}[dr]^1 \\
{\circ} \ar@{<->}[rr]_0 \ar@{<->}[ur]^1 &  & {\circ}  }}
\quad \textrm{ and } \quad 
{\xymatrix{
 & {\circ} \ar@{<->}[dr]^1 \\
{\ast} \ar@{<->}[rr]_{S^{G, i-1}_{*,*}} \ar@{<->}[ur]^1 &  & {\ast}  } }\eas along with the fact that $0 \leq \cS^{G, i-1}_{*,*} \leq 1$. This boundedness comes from the fact that $\delta_{rwd}$, $\delta_{EMD}$, and $\delta_{Haus}$ are bounded between 0 and 1, as are the constants $C_A$ and $C_S$, due to the fact that rewards are constrained to values between 0 and 1. Thus we may conclude that $1-\cA^{G,i}$ and $1-\cS^{G,i}$ are distances. 

We are unsure why these conditions were not specified for SS1. The examples the authors provided in their work satisfy these conditions, but in general they do not enforce it.

SS1 also does not specify the value of $d_{uv}$ in the case where neither $u$ or $v$ is an absorbing node. It claims, however, that the entries of $\cA^{G,i}$ and $\cS^{G,i}$ are monotonically increasing. This is not true for all values of $d_{u,v}$. In our algorithm, we universally insist that $d_{u,v}=0$ in the case where neither $u$ or $v$ is an absorbing node. The rationale for this is given in the proof of Lemma \ref{res:goodinitcond}.

\subsection{Properties of \texorpdfstring{$\cS^G$}{} and \texorpdfstring{$\cA^G$}{}}

In general we claim that the matrices $\cS^G$ and $\cA^G$ have the same properties as $\sigma_S*$ and $\sigma_A*$ in SS1, and substantiate those claims here. First we clarify the claim in \cite{wang2019} that the sequences $\{\cA_{a,b}^{G,i}\}_{i \in \N}$ and $\{\cS_{u,v}^{G,i}\}_{i \in \N}$ are monotonically increasing for all $a, b \in \Lambda$ and $u,v \in V$.

\begin{lem} \label{res:goodinitcond} There exists an initial matrix $\cS^{G,0}$ such that the matrices $\cA^{G,i}$ and $\cS^{G,i}$ are monotonically increasing. \end{lem}
\begin{proof}
Note that $1 - \cA_{\alpha,\beta}^{G,i}$ is a weighted average of a fixed reward value and the earth mover's distance of the transition probabilities of $\alpha$ and $\beta$. Then $1 - \cS_{u,v}^{G,i}$ is just the lowest highest value for some set of these weighted sums. If the entries of $1 - \cA_{\alpha,\beta}^{G,i}$ are less than the entries of $1 - \cA_{\alpha,\beta}^{G,i-1}$, then the same can be said for the entries of $1 - \cS_{u,v}^{G,i}$. Furthermore, as the transition probabilities are fixed, the entries of $1 - \cA_{\alpha,\beta}^{G,i}$ are less than the entries of $1 - \cA_{\alpha,\beta}^{G,i-1}$ if and only if the entries of $1 - \cS_{u,v}^{G,i}$ are smaller than those of $1 - \cS_{u,v}^{G,i-1}$. Therefore, it is imperative all entries of $1 - \cS_{u,v}^{G,0}$ exceed all entries of $1 - \cS_{u,v}^{G,1}$. While there are non-trivial assignments that will satisfy this for any given graph $G$, this can be universally assured by setting $d_{u,v} = 0$.
\end{proof}

\begin{cor} With the correct initial conditions, the entries of these matrices are guaranteed to converge.
\end{cor}
\begin{proof}
This is a comes directly from monotonicity and the boundedness of these values
\end{proof}

\begin{thm}
Algorithm \ref{alg:song-based} defines a distance function on the state nodes of two graph MDPs, and a pseudo-metric on the corresponding action nodes. 
\end{thm}

\begin{proof}
As Theorem 2 of SS1 does not specify that the graph must be connected. Therefore, the proof holds for this disconnected case at hand. 
\end{proof}

The only case that needs to be understood here is what happens if the graph $G  = G_M \sqcup G_N$ is the disjoint union of two identical graphs ($G_M \simeq G_N$). In this case, one initializes $S^{G,0}$ to be the identity matrix. Then for the corresponding isomorphic actions $\alpha_M$ and $\alpha_N$ in the components $G_M$ and $G_N$ respectively, $\cA_{\alpha_M, \alpha_N}^{G,1} = 1$. This is because the rewards and the transition probabilities are identical for these two functions. Therefore, if $v_M$ and $v_N$ are the source nodes for these actions, respectively, $\cS_{v_M, v_N}^{G,1} = C_S$. By similar argument, we see that these values do not change for any $i$. Therefore, the value $C_S$ denotes the discount one puts on two different states being isomorphic but distinct.

\section{Algorithm Simplification}
\label{sec:alg-simp}

In the previous section we show that the matrix $1 - \cS^G$ is a distance metric for two disconnected components $G_M \simeq G_N = G$, so the submatrix $\cS^{M, N}$ is a distance metric between the states in MDP $M$ and MDP $N$, of which we are chiefly interested. In fact, we need not compute all of $\cS^G$ in order to obtain $\cS^{M, N}$, we can simply compute it directly by only considering $u \in V_M$ and $v \in V_N$ in both \texttt{for} loops. In this case, we initialize $\cS^{M, N}$ and $\cA^{M, N}$ ($\cS^G$'s and $\cA^G$'s respective replacements in algorithm \ref{alg:song-based}) as $C_S \cdot \mathbf{I}$, where $\mathbf{I}$ is the identity matrix. 

This is justified by examining the data dependencies of each Hausdorff (as used to compute entries in $\cS^G$) and Kantorovich (as used to compute entries in $\cA^G$) operation. In the former, note that while the entirety of the $\cA^G$ matrix is passed in as the distance function, only the subset of that matrix where the rows are the actions of state $u$ and columns are the actions of state $v$ is actually used in computing the entry at $\cS^G_{u,v}$. That is to say, computing the state similarity within each sub-matrix only relies on the action similarities of actions belonging to states within that sub-matrix.

Furthermore, these entries in $\cA^G$ rely on the Kantorovich distances between the action $\alpha$ and $\beta$ from $u$ and $v$ respectively. When expanding $\alpha$ into a full probability distribution, note that the probability of transitioning from MDP $M$ to $N$ (or vice versa) via any action is $0$. Thus, those entries in the distribution have $0$ weight associated with them, and are disregarded when computing the overall distance. Thus, only the weights in $\cA^G$ corresponding to the sub-matrix in which $u$ and $v$ belong to have any bearing on the computation, and each sub-matrix can therefore be computed independently.

Regarding run-time, the big-O complexity of SS1 still holds for this simplification, as it still scales with the number of nodes in the set of state nodes and action nodes in the graph MDP; however, the memory and computation reduction from reducing the $\cS^G$ matrix ($|S_M + S_N|^2$ entries) to $\cS_{M, N}$ ($|S_M||S_N|$ entries) can be significant in practice. In essence the complexity has changed from $O(N_{iter} \cdot |S|^2 |A|^2 K_{max}^2 / \epsilon^2)$ to $O(N_{iter} \cdot |S_M||S_N| |A|^2 K_{max}^2/\epsilon^2)$, where $|S_M| < |S|$ and $|S_N| < |S|$, where $N_{iter}$ is the number of iterations the algorithms runs before converging and $K_{max} \leq |V|$ is the maximum out-degree of action nodes in $G$. 

\begin{figure}[t]
    \centering
    \includegraphics[width=0.45\textwidth]{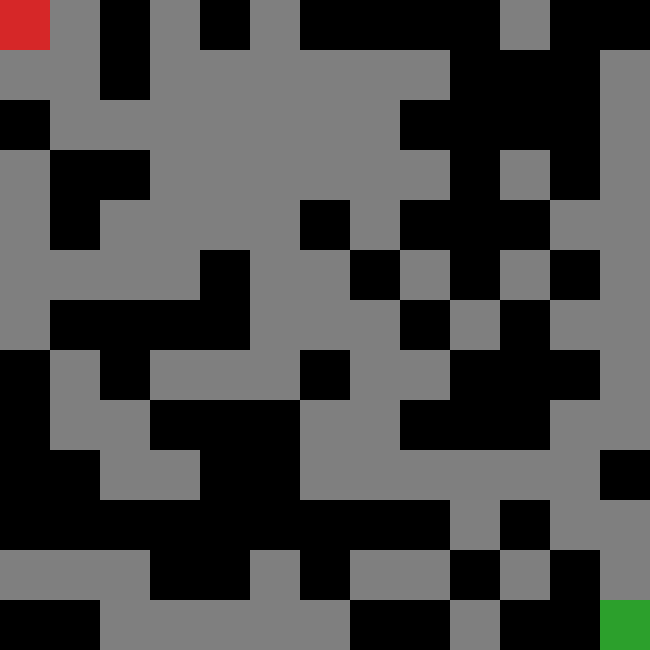}
    \caption{An example maze used in the transfer experiments. The agent begins in the upper-left red state and must navigate to the lower-right green state.}
    \label{fig:example_grid}
\end{figure}

\section{Experiment Methodology} \label{sec:experiment}

We validate SS2 by conducting several experiments wherein we assess the relationship between its similarity measurements and transfer performance. In general, our setting consists of finite MDPs representing various gridworld environments (Figure~\ref{fig:example_grid}). In each experiment, we consider a given target environment and a corpus of source environments, each environment consisting of a single start state and single goal state. We then utilize table-based Q Learning to train an agent to 90\% optimal performance, as measured by the moving average of number of steps required to complete an episode compared against the optimal path length \cite{watkins1992q}. Finally, we use the transfer methods proposed and assessed in \cite{song2016}, as well as a novel expansion of those methods that utilizes the action similarities ($\cA^{G,*}$) captured by SS2, and assess the performance of each source environment to the fixed target environment.

\subsection{Environment Generation}
We define our tasks by randomly generating a target environment and $100$ source environments. The goal of the task is to navigate the environment from the top leftmost square to the bottom rightmost square, at which point the agent is given a positive reward. Each environment generated is populated with randomly distributed obstacles, but ensures that there exists at least one path to the goal. Additionally, we repeat this generation process while modifying three independent variables to better understand how different conditions affect performance. These variables, and their associated settings, are as follows:
\begin{enumerate}
    \item Grid size: either ``Small'' (9x9) or ``Large'' (13x13)
    \item Rotations: whether or not to include random rotations when generating the source tasks. If included, then for each source task the goal state is chosen to be one of four random corners, with the start state set to the opposite corner along the diagonal. 
    \item Reward: either $1$ or $100$, associated with the goal state. All other states have no rewards or costs associated.
\end{enumerate}
This results in a total of eight experimental conditions, each with their own target task and set of $100$ source tasks. Additionally, note that in these transfer experiments, the MDPs generated are deterministic, unless otherwise specified.

\subsection{Agent Training}

We use a simple version of Q-Learning~\cite{watkins1992q, qlearnmedium} when training the initial source agents. We first set an optimality criterion by performing a breadth-first search from the environment's start state to goal state, where there is unit cost in moving to an adjacent state. We then use a learning rate $\alpha = 0.1$, discount factor $\gamma = 0.9$, and an $\epsilon$-greedy strategy starting at $\epsilon = 0.9$ to train each agent until the moving average of its last 20 episodes has reached 90\% optimal. We utilize a simple form of simulated annealing by decreasing the $\epsilon$ value by a constant decay amount of $10^{-6}$, with a minimum value of $0.1$. The agents tend to converge after approximately $5\times10^5$ steps.

\subsection{Transfer Methodology}

\begin{algorithm}[t]
\caption{Weight Transfer with Action (\texttt{T-AVG-ACT})}\label{alg:weight_action}
    \begin{algorithmic}[1]
    \REQUIRE $\delta_{\hat{S}}$, $\delta_{\hat{A}}$ 
    \REQUIRE $Q_{in} \in \mathbb{R}^{|S_{in}| \times |A_{in}|}$ 
    \REQUIRE $N_u = N_v$, $\forall u,v \in S_{in} \times S_{out} $
    
    \STATE $Q_{out} \leftarrow \mathbf{0}^{|S_{out}| \times |A_{out}}$
    \STATE $K \leftarrow EMD(U_{|S_{in}}, U_{|S_{out}|}, \delta_{\hat{S}})$
    
    \FORALL{$s_i \in S_{in}$, $s_o \in S_{out}$}
       \STATE $w \leftarrow K[s_i, s_o] / (\sum_{s_o^\prime \in S_{out}} K[s_i, s_o^\prime])$ 
       \STATE $A_{sub} \leftarrow \delta_{\hat{A}}[N_{s_i}, N_{s_o}]$
       \FORALL {$a_o \in N_{s_o}$}
          \STATE $a_m \leftarrow \argmin A_{sub}[:, a_o]$
          \IF{$A_{sub}[a_o, a_o] = A_{sub}[a_m, a_o]$}
            \STATE $a \leftarrow a_o$
          \ELSE
            \STATE $a \leftarrow a_m$
          \ENDIF
          \STATE $Q_{out}[s_o, a_o] \leftarrow w \times Q_{in}[s_i, a]$
       \ENDFOR
    \ENDFOR
    \RETURN $Q_{out}$
    \end{algorithmic}
\end{algorithm}

\begin{algorithm}[t]
\caption{State Transfer with Action (\texttt{T-STATE-ACT})}\label{alg:state_action}
    \begin{algorithmic}[1]
    \REQUIRE $\delta_{\hat{S}}$, $\delta_{\hat{A}}$ 
    \REQUIRE $Q_{in} \in \mathbb{R}^{|S_{in}| \times |A_{in}|}$ 
    \REQUIRE $N_u = N_v$, $\forall u,v \in S_{in} \times S_{out} $
    
    \STATE $Q_{out} \leftarrow \mathbf{0}^{|S_{out}| \times |A_{out}}$
    
    \FORALL{$s_i \in S_{in}$, $s_o \in S_{out}$}
       \IF{$\delta_{\hat{S}}[s_i, s_o] = \min \delta_{\hat{S}}[:, s_o]$}
          \STATE $w \leftarrow 1$
       \ELSE
          \STATE $w \leftarrow 0$
       \ENDIF
       \STATE $A_{sub} \leftarrow \delta_{\hat{A}}[N_{s_i}, N_{s_o}]$
       \FORALL {$a_o \in N_{s_o}$}
          \STATE $a_m \leftarrow \argmin A_{sub}[:, a_o]$
          \IF{$A_{sub}[a_o, a_o] = A_{sub}[a_m, a_o]$}
            \STATE $a \leftarrow a_o$
          \ELSE
            \STATE $a \leftarrow a_m$
          \ENDIF
          \STATE $Q_{out}[s_o, a_o] \leftarrow w \times Q_{in}[s_i, a]$
       \ENDFOR
    \ENDFOR
    \RETURN $Q_{out}$
    \end{algorithmic}
\end{algorithm}

\input{exported/anova}

We consider metric-method pairs in a factorial design. In these experiments, we utilize SS2, as well as the similarity metric, $d'$, defined by Song et al. in \cite{song2016} as a point of comparison, as it is similar to ours. Song's similarity metric is also computed recursively,  following Equation \ref{eq:song-metric}.
\begin{equation}
    \label{eq:song-metric}
    \begin{aligned}
    d'(s, s') = \max_{a \in A} |\mathbb{E}[r_s^a] - \mathbb{E}[r_s'^a]| 
        - C\delta_{EMD}(P_s^a, P_s'^a; d')
    \end{aligned}
\end{equation}
In Equation \ref{eq:song-metric}, $\mathbb{E}[r_s^a]$ is the expected reward for taking action $a$ at state $s$,  $P_s^a$ is the transition probability distribution for taking action $a$ at state $s$, and $C$ is a user defined constant. We also include a baseline \texttt{Uniform} metric which returns a constant distance from each source state to target state. For fairness of comparison, we ensure that the convergence criteria for each metric is identical, as both our metric and Song's utilize fixed point iterations to determine when the matrices have stopped changing. This criterion is satisfied when all pairwise comparisons between the previous iteration's distance values $d$ and current iteration's distance values $d'$ meet the following condition: $|d' - d| \leq 10^{-8} + 10^{-5}\times |d|$.

Similarly, we ensure that constants that are common to both algorithms (i.e. Song's analogue of $C_A$) have the same value. In this case, we set $C_A=0.5$. We also initialize $C_S = 0.9995$, but note that $d'$ has no corresponding constant.

We use methods described in Song's work, namely \texttt{Weight} and \texttt{State} transfer~\cite{song2016} to transfer agent knowledge from one task for use in another. The \texttt{Weight} transfer method (not to be confused with neural network weights) initializes the Q function, or value function, of the target task by using the values of the Q-function (or value function) in the source task(s). Each contribution is weighed by the corresponding entry of the optimal transport matrix computed by the Kantorovich metric upon the State-State distance matrix. As stated in \ref{sec:intro}, we refer to this method as \texttt{T-AVG}. The \texttt{State} transfer, which we will henceforth refer to as \texttt{T-STATE} is similar, but simply sets the Q-function (or value function) for each state by selecting the entry of the source task corresponding to the state that is least distant to the target state. That is, only the most similar source state to each target state will contribute to its initialization. 

We then extend these methods by using the action similarity matrix, ($\cA^G$) to better inform mappings between actions in each source-target state pair, as described in Algorithm \ref{alg:weight_action} (\texttt{T-AVG-ACT}). In performing each transfer experiment, we utilize the specified metric to obtain a distance matrix $\delta_{\hat{S}}$ between states in each source environment and the fixed target environment, as well as a corresponding distance matrix $\delta_{\hat{A}}$ between actions.

Intuitively then, Algorithm \ref{alg:weight_action} can be seen as assigning the minimally distant source action to each target action, breaking ties by the original action ordering alignment. Note that this action-space correspondence precondition is required by Song's metric, but we are only using it as a heuristic to break ties. Algorithm \ref{alg:state_action} (\texttt{T-STATE-ACT}) is quite similar, except it weighs only the least distant State-State pair, similar to \texttt{T-STATE}.

Finally, we then use each of the transfer methods to initialize the Q table for the target environment. Note that while all four methods can support transferring information from multiple source tasks to a single target task, we only consider the one-to-one case in this experiment. We train the resulting agents for $10^3$ steps, with a fixed $\epsilon$ value of $0.1$ (other hyper-parameters are held the same as when training the source tasks). We repeat this procedure for $50$ trials, to reduce variance that results from $\epsilon$-greedy exploration, and store each resulting performance curve.

Note that the transfer variants that use the action similarity matrix are only applicable to our algorithm, since $d'$ only computes state similarity. This results in a total of eight metric-method combinations.

\section{Results}

\subsection{Transfer Performance}

\begin{figure}
\includegraphics[width=0.48\textwidth]{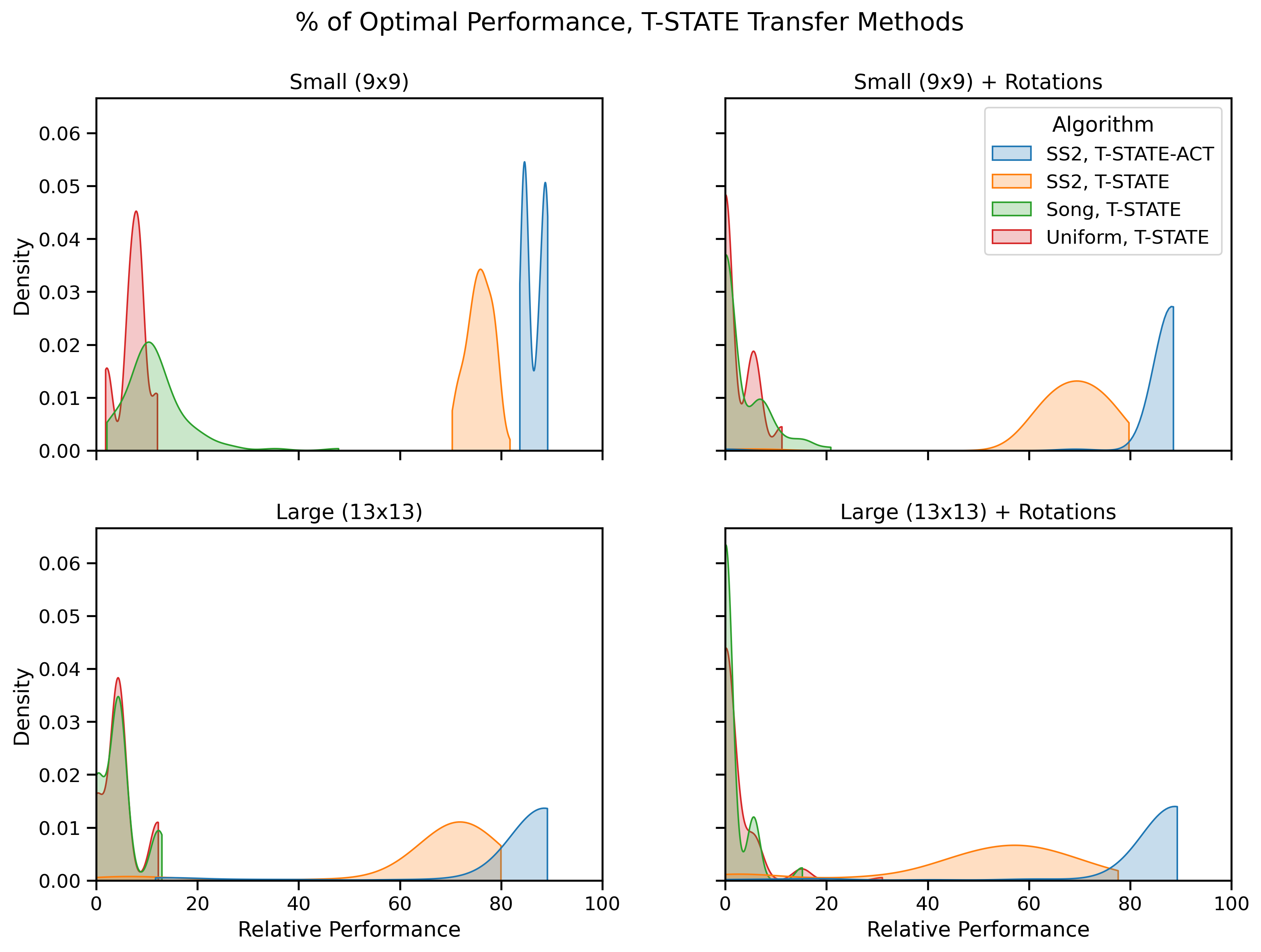}
\includegraphics[width=0.48\textwidth]{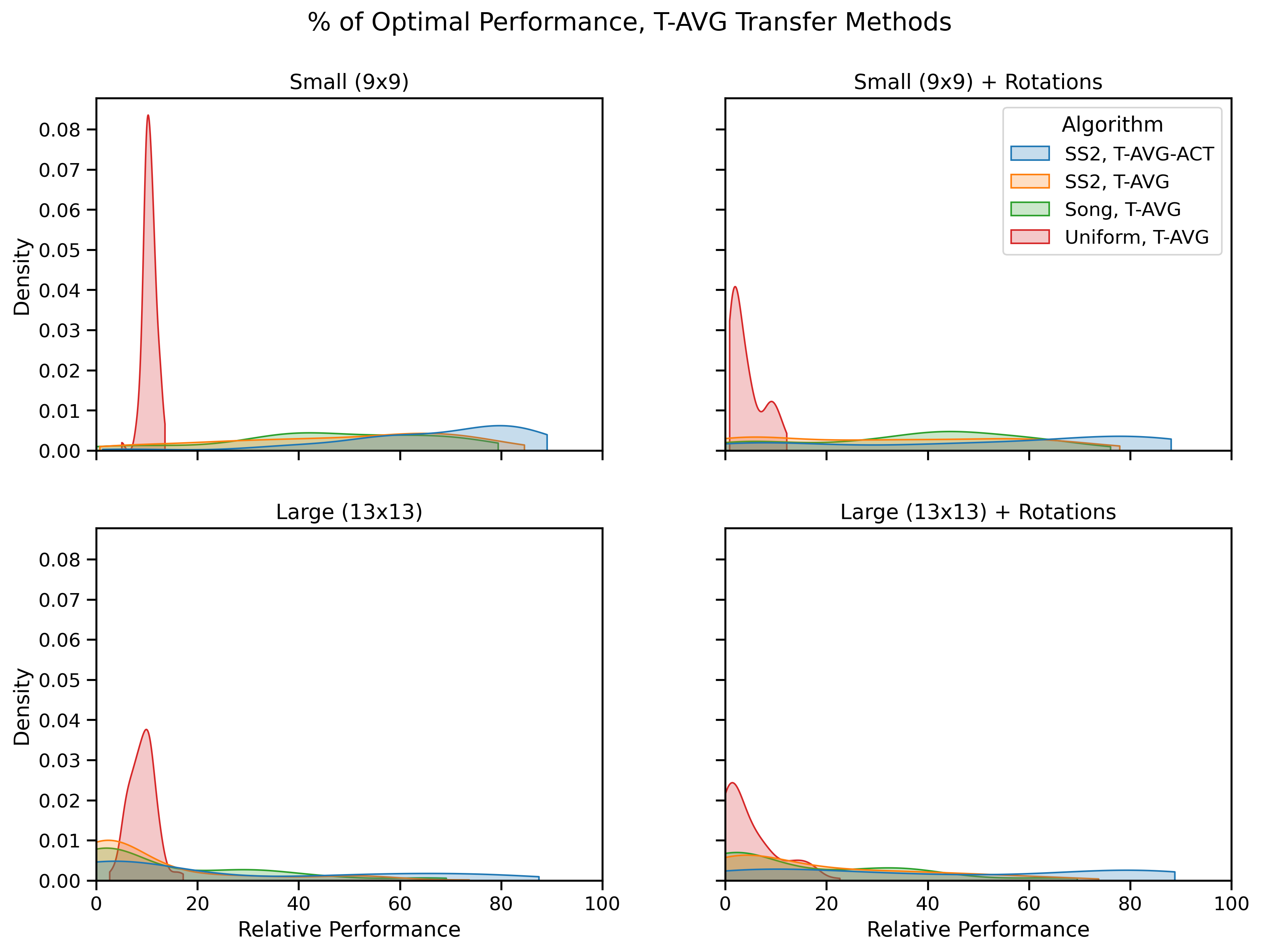}
\caption{Distributions of agent performance over all source tasks, averaged over 50 trials. The x-axis represents the average number of episodes completed by an agent in the first 2500 steps of training divided by the number of episodes completed by an optimal agent in the same number of steps, which we call \texttt{Relative Performance}. In essence, farther right indicates closer to optimal performance.}
\centering
\label{fig:dist_rel_best}
\end{figure}

We begin by calculating the \texttt{Average Episodes Completed} across the $50$ trials for each (source, metric-method, experiment) tuple, which is simply the mean number of episodes the agents were able to complete in a chosen measurement period (in our case, the first $2500$ steps of training). This enables us to compute an \texttt{Episode Performance} which is the average number of steps per episode observed (i.e. $2500$ $/$ \texttt{Average Episodes Completed}). We then calculate a \texttt{Relative Performance} by comparing this \texttt{Episode Performance} to the optimal step count in the target environment, via \texttt{Optimal Length} $/$ \texttt{Episode Performance}, and converting to a percentage by multiplying by $100$.

We conduct an Analysis of Variance (ANOVA) test for each of the three independent factors used in creating the environments \cite{st1989analysis}. In these tests, the dependent variable was the \texttt{Relative Performance}.
The Bonferroni corrected p-values can be seen in Table \ref{tab:anova_res} \cite{bonferroni1936teoria}. As shown, for most metric-method pairs, both grid size (\texttt{Dimension}) and the presence of rotations (\texttt{Rotate}) were statistically significant in affecting transfer performance, while the reward amount (\texttt{Reward}) was not significant for any pair. Thus, in the remainder of our analyses and visualizations, we focus only on grid size and presence of rotations, using the results in the reward $ = 100$ experiments to most directly compare with \texttt{T-AVG} and \texttt{T-STATE}.

The full results consisting of the mean and standard deviations for each metric-method in each experiment can be seen in Table \ref{tab:full_results}. 
Interestingly, across all conditions for SS2, \texttt{T-STATE-ACT} dominated \texttt{T-AVG-ACT} and \texttt{T-STATE} dominated \texttt{T-AVG}. However, for \texttt{Uniform} and $d'$ similarities, \texttt{T-AVG} transfer dominated \texttt{T-STATE} transfer.

We thus choose to highlight these four metric-method pairs, as they are the best methods for each metric:
\begin{enumerate}
    \item \texttt{SS2} Metric, \texttt{T-STATE-ACT} Transfer
    \item \texttt{SS2} Metric, \texttt{T-STATE} Transfer
    \item $d'$ Metric, \texttt{T-AVG} Transfer 
    \item \texttt{Uniform} Metric, \texttt{T-AVG} Transfer
\end{enumerate}


These results are highlighted in Figures \ref{fig:dist_rel_best} and \ref{fig:curves_best}. Figure \ref{fig:dist_rel_best} highlights the percent of optimal performance able to be achieved in the first 2500 iterations, as a density plot over the 100 source tasks. Figure \ref{fig:curves_best} displays the cumulative number of episodes completed over time, averaged over the 100 source tasks. 
Overall we observe that SS2, with both the \texttt{T-STATE-ACT} and \texttt{T-STATE} transfer methods, significantly outperforms $d'$, both in terms of consistency across the source tasks as well as transfer performance attained. We suspect that this is due to the action distances effectively allowing our algorithm to leverage states which are similar despite their out-actions being misaligned (e.g. states that are one move away from the goal should be similar even if the action to perform that move is different).

\begin{figure*}
\includegraphics[width=\textwidth]{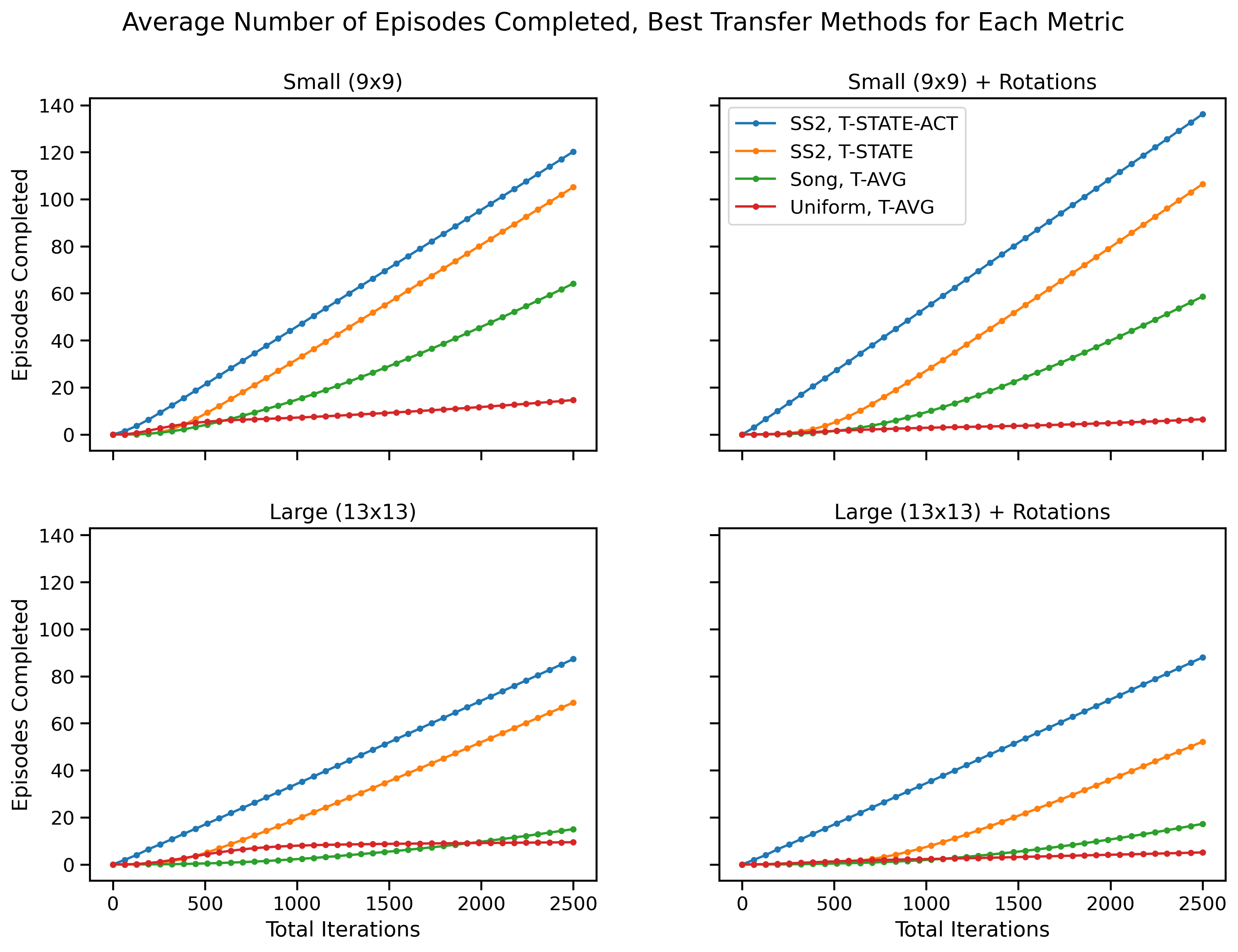}
\caption{Performance curves of best transfer methods for each metric over time. Data points are averaged over 50 trials and 100 source tasks.}
\centering
\label{fig:curves_best}
\end{figure*}

\section{Discussion}

\subsection{Structural Similarity; MDP Distance} \label{sec:dists}

Ideally, we would want a single value to characterize the overall similarity between two MDPs, a straightforward way to do this is as follows:

Using Algorithm \ref{alg:song-based} or a simplified version, assume we may now compute the function $\delta_{\hat{S}} = 1 - \sigma_{\hat{S}} = 1 - \sigma_{S_M \cup S_N} = 1 - \cS^{G, *}$, for MDPs $M=(S_M, A, P_M, R_M)$ and $N=(S_N, A, P_N, R_N)$. We analogously define $\delta_{\hat{A}} = 1 - \cA^{G, *}$. We may now apply the measures used in \cite{song2016} to compute a final similarity measure between the MDPs, i.e. the Hausdorff and Kantorovich metrics. Assuming $M$, $N$ and $S$ defined as stated previously, we state the following definitions:

\begin{dfn}
The \emph{Hausdorff structural similarity distance}, $\Phi$, of $M$ and $N$ is defined as
$$
    \Phi(M, N) = \delta_{Haus}(V_M, V_N, \delta_{\hat{S}}), 
$$
where $V_M$ and $V_N$ are the state nodes of the MDP graphs for $M$ and $N$, respectively.
\end{dfn}
\begin{dfn}
The \emph{Kantorovich structural similarity distance}, $\Psi$, is defined as
$$
    \Psi(M, N) = \delta_{EMD}(U_{|S_M|}, U_{|S_N|}, \delta_{\hat{S}}), 
$$
where $U_{a}$ is a discrete Uniform distribution with $p=\frac{1}{a}$. 
\end{dfn}

We tested these measures using the Q-learning framework we used for \ref{sec:experiment}, but we did not find correlation between transfer performance and MDP distances (for more details, see Appendix~\ref{sec:mdp_dist_appdx}. However, because the state distances computed using SS2 seem to be useful in terms of improving transfer from one task to another, we believe the issue is in the reduction technique, i.e. the use of the Kantorovich and Hausdorff distances. Discovering other techniques to reduce the information from $\cS^{G,*}$ into a single value may be a possible avenue of future research, but another potentially more fruitful avenue might be how to better leverage the similarity information contained in $\cS^{G,*}$ and $\cA^{G,*}$ to improve transfer in other RL tasks and algorithms. 

\subsection{Limitations}
Our results here suggest that using $\cS^{G,*}$ and $\cA^{G,*}$ as similarity characterizations to improve transfer between RL tasks has potential, but we acknowledge it currently has various limitations. The most obvious is that SS2 is limited to finite MDPs that can be fully specified (i.e. complete knowledge of the transition and reward functions), and even in this case, computing $\cS^{G,*}$ and $\cA^{G,*}$ can be get very expensive as the size of the state and action spaces increase. 

Another limitation we found in our experiments is that results are strongly dependent on how one defines the MDP for a given task. For example, DRL often assumes that every action is available to the agent at every step, which is not always the case. Choosing to include action nodes for every action at every state (where unavailable actions would be modeled by say, returning to the current state with probability 1.0) will produce different similarity values than omitting actions nodes where the actions are not available. This is a modeling choice for how to deal with obstacles in the gridworld, but has a significant effect on the similarity values returned by SS2.  

\subsection{Future Work}

Future work of interest on this subject could include performing additional experiments to better characterize the relationship between the state similarities returned by SS2 and transfer performance between tasks. In particular, we would find experiments that include MDP environments other than the traditional gridworld of interest. Another avenue would be assessing transfer in the continual learning setting \cite{ring1994continual, khetarpal2020continual}, even with simple gridworld tasks. How might we leverage similarity measures such as those from SS2 to continuously modify an RL agent over a lifetime of performing a set of tasks? Finally, an area of particular interest would be to extend SS2 to operate on larger MDPs through sampling techniques, such those used by \cite{castro2019scalable, castro2021mico}.

\section{Conclusion}
We proposed the Structural Similarity for two MDPs, or SS2, algorithm for computing state similarities between two MDPs. We proved that the result of SS2 can be used to define a distance metric using results from \cite{wang2019}, upon which SS2 is based. We then showed that by leveraging the state-to-state similarity results from SS2 with RL transfer algorithms from \cite{song2016}, we could improve transfer between tasks in a gridworld navigation environment over those shown in \cite{song2016} for the same tasks. Furthermore, by incorporating action similarity information from SS2 into the transfer algorithm, we further improved transfer performance. 

We also tested task-to-task similarity measures comprising the SS2 outputs and the Hausdorff and Kantorovich distance metrics, and found that transfer performance did not correlate with those similarity scores. Therefore, we conclude that, while the state-to-state similarity values produced by SS2 between two MDPs is useful for transfer, reducing those values into a single scalar one, using the Hausdorff and Kantorovich distances in particular, do not produce useful measures of task similarity for our environment and RL algorithm. Future work should include additional verification of these results with additional tasks and algorithms, and investigate additional uses of information in SS2 for RL tasks. Other natural extensions could include attempting to apply SS2 to larger MDPs via sampling and graph reduction as well as extending SS2 to return state-to-state similarities for any finite number of tasks represented as finite MDPs. 

\section*{Acknowledgements}
This work was funded by the DARPA Lifelong Learning Machines (L2M) Program.  The views, opinions, and/or findings expressed are those of the author(s) and should not be interpreted as representing the official views or policies of the Department of Defense or the U.S. Government.

\printbibliography

\pagebreak
\clearpage

\appendices

\section{MDP Distance and Transfer Correlation} \label{sec:mdp_dist_appdx}

To assess the effectiveness of the MDP Distance, as discussed in Section \ref{sec:dists}, we computed the Pearson's correlation coefficient, for all metric-method combinations in their individual conditions, as well as across all conditions combined \cite{benesty2009pearson}. These results can be seen in full in Tables \ref{tab:corr_res} and \ref{tab:corr_haus_res}, regarding the final Kantorovich and Hausdorff distances respectively. We desired a strong negative correlation (i.e. greater distance implying worse performance). There were two key takeaways from performing this analysis:
\begin{enumerate}
    \item No metric-method resulted in having the desired negative correlation across all conditions combined over all sizes, rewards, and rotations, using either of the two final distance scores. The implication here is that while the distance matrices can effectively be used to perform transfer, they do not seem to be predictive of transfer in the sense of simply ``higher distance'' $\implies$ ``worse transfer''. This is likely because the distances observed in each experiment tended to group together (e.g. having a larger reward in Song's metric would lead to a larger distance, scaled linearly, but not necessarily any difference in transfer performance).
    \item In individual experiments, some metric-method combinations demonstrated correlation, although note again that no combination had a consistent correlation across the conditions. One reason why our metric with its best method (i.e. state and action transfer) did not exhibit correlation could be the fact that regardless of the distance, the transfer performance was quite good. This highlights the need to further experiment in more complex scenarios or perhaps different environments altogether, which we hope to explore in the near future.
\end{enumerate}

\input{exported/full_res}

\input{exported/corr}

\input{exported/corr_haus}

\end{document}

%% file: exported/anova.tex
\begin{table}
\centering
\caption{Anova Results. Reported p-values are Bonferroni corrected; Reward is omitted as its p-value was 1.0000 in all cases.}
\label{tab:anova_res}
\begin{tabular}{lrr}

{} &  Dimension P-Value &  Rotate P-Value \\
Algorithm        &                    &                 \\

SS2, T-STATE-ACT &         4.9812e-02 &      1.0000e+00 \\
SS2, T-AVG-ACT   &         3.9507e-30 &      1.0000e+00 \\
SS2, T-STATE     &         1.1683e-30 &      1.2733e-22 \\
SS2, T-AVG       &         1.2963e-48 &      6.3088e-01 \\
Song, T-STATE    &         3.8663e-23 &      3.1010e-44 \\
Song, T-AVG      &         2.1109e-63 &      3.6490e-01 \\
Uniform, T-STATE &         8.9806e-07 &      3.7355e-45 \\
Uniform, T-AVG   &         1.0000e+00 &      5.6624e-73 \\

\end{tabular}
\end{table}

%% file: exported/full_res.tex
\begin{table*}
\centering
\caption{Full Statistical Results. Entries are "mean (std)" for the given algorithm and condition.}
\label{tab:full_results}
\begin{tabular}{lllllllll}

{} &     SS2, AVG & SS2, AVG-ACT &   SS2, STATE & SS2, STATE-ACT &    Song, AVG & Song, STATE & Uniform, AVG & Uniform, STATE \\
Condition     &              &              &              &                &              &             &              &                \\

Lg, R1        &  11.3 (16.6) &  26.1 (29.9) &  66.0 (17.4) &    84.2 (14.9) &  14.6 (18.2) &   4.3 (3.9) &    9.0 (2.5) &      4.6 (3.8) \\
Lg, R1, Rot   &  20.8 (21.2) &  42.2 (33.0) &  51.0 (19.0) &    85.4 (14.8) &  16.0 (18.8) &   1.5 (3.2) &    5.0 (5.5) &      1.6 (3.4) \\
Lg, R100      &  11.1 (17.0) &  25.9 (29.9) &  66.1 (17.5) &    83.8 (16.4) &  14.4 (18.5) &   4.3 (3.9) &    9.1 (2.6) &      4.6 (3.8) \\
Lg, R100, Rot &  18.9 (19.8) &  42.4 (32.7) &  50.1 (19.5) &    84.5 (16.1) &  16.6 (18.9) &   1.4 (3.1) &    4.9 (5.4) &      1.9 (4.5) \\
Sm, R1        &  47.9 (22.3) &  67.1 (18.4) &   75.7 (2.5) &     86.6 (2.0) &  45.9 (21.0) &  11.6 (6.6) &   10.5 (1.2) &      7.2 (2.8) \\
Sm, R1, Rot   &  33.1 (24.5) &  52.0 (30.8) &  68.0 (10.7) &     86.9 (9.6) &  37.4 (20.7) &   2.7 (4.6) &    4.1 (3.2) &      2.2 (3.2) \\
Sm, R100      &  47.9 (22.3) &  67.1 (18.5) &   75.7 (2.5) &     86.6 (2.0) &  46.2 (20.7) &  11.6 (6.6) &   10.5 (1.3) &      7.2 (2.9) \\
Sm, R100, Rot &  32.9 (24.5) &  51.9 (30.6) &  68.1 (10.6) &     87.2 (9.0) &  37.5 (20.9) &   2.8 (4.6) &    4.1 (3.2) &      2.3 (3.3) \\

\end{tabular}
\end{table*}

%% file: exported/corr.tex
\begin{table*}
\centering
\caption{Kantarovich Pearson Correlation Results: Desired relation is negative.}
\label{tab:corr_res}
\begin{tabular}{lrrrrrrrr}

{} &  SS2, STATE-ACT &  SS2, AVG-ACT &  SS2, STATE &  SS2, AVG &  Song, STATE &  Song, AVG &  Uniform, STATE &  Uniform, AVG \\
Condition     &                 &               &             &           &              &            &                 &               \\

All           &           0.099 &         0.138 &       0.198 &     0.120 &       -0.011 &      0.117 &           0.211 &         0.095 \\
Lg, R1        &           0.224 &        -0.083 &       0.168 &    -0.109 &       -0.081 &      0.025 &          -0.059 &         0.205 \\
Lg, R1, Rot   &          -0.159 &         0.084 &      -0.148 &    -0.088 &        0.174 &     -0.227 &           0.024 &        -0.015 \\
Lg, R100      &           0.148 &        -0.072 &       0.140 &    -0.123 &       -0.084 &     -0.059 &          -0.060 &         0.190 \\
Lg, R100, Rot &          -0.142 &         0.102 &      -0.135 &    -0.072 &        0.169 &     -0.264 &           0.032 &        -0.010 \\
Sm, R1        &           0.156 &        -0.263 &       0.099 &    -0.343 &       -0.408 &     -0.371 &          -0.115 &         0.230 \\
Sm, R1, Rot   &           0.208 &        -0.113 &       0.149 &    -0.161 &       -0.488 &     -0.382 &          -0.183 &        -0.021 \\
Sm, R100      &           0.145 &        -0.261 &       0.113 &    -0.344 &       -0.408 &     -0.378 &          -0.119 &         0.216 \\
Sm, R100, Rot &           0.237 &        -0.114 &       0.156 &    -0.163 &       -0.474 &     -0.392 &          -0.172 &        -0.023 \\

\end{tabular}
\end{table*}

%% file: exported/corr_haus.tex
\begin{table*}
\centering
\caption{Hausdorff Pearson Correlation Results: Desired relation is negative.}
\label{tab:corr_haus_res}
\begin{tabular}{lrrrrrrrr}

{} &  SS2, STATE-ACT &  SS2, AVG-ACT &  SS2, STATE &  SS2, AVG &  Song, STATE &  Song, AVG &  Uniform, STATE &  Uniform, AVG \\
Condition     &                 &               &             &           &              &            &                 &               \\

All           &           0.014 &        -0.046 &      -0.056 &    -0.083 &       -0.042 &     -0.020 &           0.211 &         0.095 \\
Lg, R1        &           0.017 &         0.015 &      -0.001 &    -0.039 &       -0.013 &      0.030 &          -0.059 &         0.205 \\
Lg, R1, Rot   &          -0.101 &         0.090 &      -0.160 &     0.044 &        0.225 &     -0.186 &           0.024 &        -0.015 \\
Lg, R100      &           0.076 &         0.016 &       0.086 &    -0.007 &       -0.016 &      0.051 &          -0.060 &         0.190 \\
Lg, R100, Rot &          -0.141 &         0.074 &      -0.146 &    -0.031 &        0.216 &     -0.228 &           0.032 &        -0.010 \\
Sm, R1        &           0.260 &        -0.323 &       0.158 &    -0.222 &       -0.233 &     -0.275 &          -0.115 &         0.230 \\
Sm, R1, Rot   &           0.150 &        -0.230 &       0.201 &    -0.188 &       -0.431 &     -0.244 &          -0.183 &        -0.021 \\
Sm, R100      &           0.235 &        -0.320 &       0.171 &    -0.223 &       -0.233 &     -0.266 &          -0.119 &         0.216 \\
Sm, R100, Rot &           0.162 &        -0.229 &       0.200 &    -0.186 &       -0.425 &     -0.263 &          -0.172 &        -0.023 \\

\end{tabular}
\end{table*}